%% file: main_arxiv.tex
\documentclass[11pt]{article}
\usepackage[utf8]{inputenc}

\usepackage{preamble}
\usepackage{notations}

\title{Uniform Stability for First-Order Empirical Risk Minimization}

\author{%
    Amit Attia%
    \thanks{\scriptsize Blavatnik School of Computer Science, Tel Aviv University; \texttt{amitattia@mail.tau.ac.il}.}
    \and
    Tomer Koren%
    \thanks{\scriptsize Blavatnik School of Computer Science, Tel Aviv University, and Google Research Tel Aviv; \texttt{tkoren@tauex.tau.ac.il}.}
}

\begin{document}

\maketitle

\begin{abstract}%
\input{abstract.tex}
\end{abstract}

\input{sec_introduction}

\input{sec_preliminaries}

\input{arxiv/sec_l2_arxiv}

\input{arxiv/sec_md_arxiv}

\subsection*{Acknowledgements}

This work has received support from the Israeli Science Foundation (ISF) grant no.~2549/19, the Len Blavatnik and the Blavatnik Family foundation, the Deutsch Foundation, and the Yandex Initiative in Machine Learning.

\bibliographystyle{abbrvnat}
\bibliography{references}

\appendix

\input{arxiv/appendix_arxiv}

\end{document}

%% file: abstract.tex
We consider the problem of designing uniformly stable first-order optimization algorithms for empirical risk minimization.
Uniform stability is often used to obtain generalization error bounds for optimization algorithms, and we are interested in a general approach to achieve it.
For Euclidean geometry, we suggest a black-box conversion which given a smooth optimization algorithm, produces a uniformly stable version of the algorithm while maintaining its convergence rate up to logarithmic factors.
Using this reduction we obtain a (nearly) optimal algorithm for smooth optimization with convergence rate $\otil(1/T^2)$ and uniform stability $O(T^2/n)$, resolving an open problem of \citet{chen2018stability,attia2021algorithmic}.
For more general geometries, we develop a variant of Mirror Descent for smooth optimization with convergence rate $\otil(1/T)$ and uniform stability $O(T/n)$, leaving open the question of devising a general conversion method as in the Euclidean case.

%% file: sec_introduction.tex
\section{Introduction}

We consider a canonical problem in machine learning: empirical risk minimization using first-order convex optimization.
Given a training sample $S=(z_1,\ldots,z_n)$ of $n$ instances, the goal is to minimize the empirical risk $\emprisk(x) \eqdef \tfrac{1}{n}\sum_{i=1}^n \ell(x;z_i)$ where $\ell(\cdot,z)$ is a convex loss function.
Our focus is on the smooth case which contains a variety of first-order algorithms including gradient descent (GD) and Nesterov's celebrated accelerated gradient method \citep{nesterov1983method}.
In statistical learning, the empirical risk is used as a proxy and our true goal is to minimize the \textit{population risk} $\risk(x) \eqdef \E_{z \in \D}\brk[s]{\ell(x;z)}$ of an unknown distribution $\D$.
The performance of a learning algorithm $\alg$ is evaluated by the expected \textit{excess population risk},
\begin{align*}
    \E_{S}\brk[s]{\risk(A(S)) - \risk(\xstar)},
\end{align*}
where $\xstar \in \argmin \risk(x)$,
which is often bounded by managing the trade-off between 
the expected optimization and generalization errors:
\begin{align*}
    \E_S[\risk(A(S)) - \risk(\xstar)]
    &=
    \E_S[\underbrace{\risk(A(S))-\emprisk(A(S))}_{\text{generalization}}]+\E_S[\underbrace{\emprisk(A(S))-\emprisk(\xstar)}_{\text{optimization}}]
    .
\end{align*}
Thus, for a given optimization method to minimize the empirical risk, we are often interested in bounding the generalization error of its solution.
A fundamental framework for obtaining such bounds is algorithmic stability~\citep{bousquet2002stability,shalev2009stochastic}.

Algorithmic stability has emanated as a central tool for generalization analysis of learning algorithms.
The pioneering work of \citet{bousquet2002stability} introduced the notion of \textit{uniform stability}, arguably the most common notion of algorithmic stability in learning theory.
Essentially, to this day, stability analysis is the only general
approach for obtaining tight, dimension free generalization bounds for convex optimization algorithms applied to the empirical risk (see
\citealp{shalev2009stochastic,feldman2016generalization}).

Although stability analysis has proved to be an effective tool for generalization bounds, it unfortunately applies for specific combinations of algorithms, objectives and geometries.
For example, gradient descent is known to be uniformly stable for smooth objectives in $\ell_2$ geometry~\citep{hardt2016train,feldman2018generalization,chen2018stability}, but the stability analysis fails in other geometries (e.g., $\ell_1$) due to a lack of the contractivity property for gradient steps~\citep{asi2021private}. 
Another example is Nesterov's gradient method, for which uniform stability is quadratic in the number of steps over quadratic objectives \citep{chen2018stability} but for general smooth objectives grows exponentially fast~\citep{attia2021algorithmic}.

The last example is particularly intriguing: as highlighted by \citet{chen2018stability} and \citet{attia2021algorithmic}, it is currently not known whether there exists an optimal method for smooth optimization, with convergence rate $O(1/T^2)$, which also exhibits the optimal uniform stability rate $O(T^2/n)$ for general smooth objectives.  This is an important issue as momentum-based methods, inspired by Nesterov's optimal method, are being extensively used for empirical risk minimization in practice and understanding their stability properties would help in shedding light on the generalization ability of such methods.
Developing an optimal and uniformly stable method (or proving that one does not exist) was thus left as an open problem by \citet{chen2018stability,attia2021algorithmic}.

\subsection{Contributions}

In this paper, motivated by the open question of \citet{chen2018stability,attia2021algorithmic}, we study general techniques for uniformly stable empirical risk optimization with smooth and convex objectives.  First, we focus on the Euclidean case and give a general and widely-applicable technique for converting optimization algorithms to uniformly stable ones.  Then, we move on to develop uniformly stable algorithms for smooth convex optimization in more general normed spaces.

\paragraph{General reduction in Euclidean geometry.}
We provide an algorithm, $\stabregconvex$ (see \cref{alg:l2}), which performs a black-box conversion from a given optimization algorithm for convex, smooth and Lipschitz objectives to a uniformly stable algorithm with nearly the same convergence rate. 
Following is an informal version of our first main result (stated formally in \cref{thm:main_result1}).

\begin{theorem}[informal]
Assume an optimization algorithm $\alg$ with convergence rate $O\brk{T^{-\gamma}}$ over convex, smooth and Lipschitz functions w.r.t.~the Euclidean norm. Then applying $\stabregconvex$ to $\alg$ yields an algorithm with convergence rate $\otil\brk{T^{-\gamma}}$ whose $T$'th iterate is $O\brk{T^\gamma/n}$-uniformly stable.
\end{theorem}

Thus, the conversion preserves the rate of $\alg$ (up to logarithmic factors) and exhibits essentially the best convergence vs.\ stability trade-off one could hope for: indeed, any improvement to one of the rates (without compromising the other) would lead to a contradiction to statistical lower bounds (this is discussed is detail by \citealp{chen2018stability}).%
\footnote{We remark that the setting of \citet{chen2018stability} did not include a Lipschitz assumption (in addition to smoothness), as they discuss in their Section 4.2. That said, a straightforward modification of the proof of Theorem 7 in \citet{chen2018stability} can accommodate the Lipschitz assumption by a simple scaling of the loss function.}
Applying this result to Nesterov's accelerated gradient method, we obtain an algorithm with (nearly) \emph{optimal} convergence rate $\otil(1/T^2)$ and a matching \emph{optimal} stability rate of $O(T^2/n)$, resolving an open problem posed by \cite{chen2018stability} and reiterated by \cite{agarwalAKTZ20} and \citet{attia2021algorithmic}.

Our conversion procedure is based on two simple observations.
The first is that the minimizer of a strongly convex objective is uniformly stable.
The second is that with smoothness, converging very close to the minimizer of a regularized objective comes almost for free (at the cost of only a logarithmic factor) since strong convexity and smoothness together allow for linear rates of convergence.
Carefully combining the two leads to a simple yet effective way to achieve stability in smooth convex optimization with a minimal degradation in convergence rate.

\paragraph{Stable Mirror Descent for general norms.}

We move on to address uniform stability in more general normed spaces.
A general approach to optimization with general norms is the so called Mirror Descent \citep{nemirovskij1983problem} which has convergence rate $O(1/T)$ for smooth objectives (e.g.,~\citealp{bubeck2015convex}) and accelerated variants with rate $O(1/T^2)$ \citep{tseng2008accelerated,allen2017linear}.
A natural followup question to our investigation in the Euclidean case is whether there exists a variant of (accelerated) Mirror Descent with a similar convergence rate, which is also uniformly stable.

This general scenario poses additional challenges, as even for simple non-accelerated Mirror Descent, previous work by \citet{asi2021private} gave an indication that a standard Mirror Descent gradient step fails to be contractive (in fact, it is slightly expansive). This questions the primary approach for proving stability of iterative methods in the context of Mirror Descent.

As it turns out, a general conversion scheme as the one we use in \cref{alg:l2} does not easily extends to general geometries. 
In a nutshell, the issue is the following:
given a black-box algorithm for smooth and convex optimization, the standard reduction for obtaining a linear rate assuming the function is also strongly convex is based on a contraction argument relating the convergence rate upper bound to the squared distance from the minimizer, via strong convexity; for general norms, however, the convergence (e.g., of Mirror Descent) depends in general on the Bregman divergence rather than the squared distance and the argument does not go through.

Instead, we devise a specialized algorithm called $\stabregrel$ (\cref{alg:stable_md}), which obtains the following result (stated formally at \cref{thm:main_result2}), and leave the problem of designing a generic conversion method for general geometries as an open question.

\begin{theorem}[informal]
    Assume a convex loss function $\ell$, which is smooth and Lipschitz (w.r.t.~a norm  $\norm{\cdot}$) and a $1$-strongly convex regularization $R$ (w.r.t.~the same norm).
    Then $\stabregrel$ applied to the empirical risk with regularizer $R$ has convergence rate $\otil\brk{1/T}$ and its output after $T$ steps is $O\brk{\ifrac{T}{n}}$-uniformly stable.
\end{theorem}
Hence, $\stabregrel$ has the desired uniform stability of $O(T/n)$ and has nearly the same rate as Mirror Descent.
The algorithm is also based on the simple approach of optimizing a regularized objective; however,
as regularization in general norms can impair smoothness (e.g., $\tfrac{1}{2}\pnorm{\cdot}^2$ for $1<p<2$ is not smooth, see \cref{sec:non_smooth_regulator}), our analysis is based on relative smooth convex optimization \citep{lu2018relatively} (the regularization is smooth w.r.t.~itself), exploiting the linear rate to converge to a stable minimizer.

\paragraph{Examples.} We discuss two example applications of our general algorithm.

\begin{itemize}[leftmargin=*]
\item \emph{$\ell_p$ geometry $(1 < p \leq 2)$}:
In the case where
$\xdomain$ is the $\ell_p$ unit ball
and $\ell(\cdot,z)$ is convex, smooth and Lipschitz w.r.t.~$\ell_p$ norm for all $z$,
applying \cref{thm:main_result2} with the mirror map $\reg(x) = \ifrac{\pnorm{x}^2}{2(p-1)}$,
$\stabregrel$
has convergence rate
$\otil\brk*{\ifrac{1}{(p-1)T}}$,
and is $O\brk*{\ifrac{T}{n}}$-uniformly stable.

\item \emph{$\ell_1$ geometry}:
In the case where $\xdomain=\brk[c]{x \in \R_{\geq 0}^d : \norm{x}_1=1}$ and $\ell(\cdot,z)$ is convex, smooth and Lipschitz w.r.t.~$\ell_1$ norm for all $z$,
applying \cref{thm:main_result2} with negative entropy as the mirror map $\reg(x) = \sum_{i=1}^d x_i \log x_i$, the convergence rate is
$\otil\brk*{\ifrac{\log(d)}{T}}$,%
\footnote{Throughout, logarithmic factor of the space dimension are \emph{not} suppressed by the $\otil$ notation.}
and the algorithm is $O\brk*{\ifrac{T}{n}}$-uniformly stable.
\end{itemize}

\paragraph{Open problems.}
A couple of interesting questions remain open for investigation.
The first is whether we can remove the additional $\log$ factors from both the Euclidean and general geometry methods, thus obtaining stability ``for free.''
The second is whether we can devise a general conversion method, analogous to the one we developed in the Euclidean case, that would apply to more general norms.

\subsection{Related work}

Classical generalization theory appealed to uniform convergence of the empirical risk to the population risk.
Without further assumptions on convex functions, the rate of uniform convergence for stochastic convex optimization is dimension-dependent and lower bounded by $\Omega(\sqrt{d/n})$ \citep{shalev2010learnability,feldman2016generalization}.
Using a stability analysis, recent progress was made on stochastic optimization generalization bounds of convex risk minimizers, starting with the influential work of \citet{bousquet2002stability} and \citet{shalev2009stochastic}.
A variety of notions of algorithmic stability  exists in the literature and differ in the distance measure and aggregation of multiple changes
\citep{bousquet2002stability,mukherjee2006learning,shalev2010learnability,london2017pac,lei2020fine}.
Data dependent generalization bounds based on stability arguments were also studied by \citet{maurer2017second,kuzborskij2018data}.
Recent work derived tighter bounds of generalization from stability
\citep{feldman2018generalization,feldman2019high,bousquet2020sharper,klochkov2021stability},
and the approach of stability analysis has been influential in a variety of settings (e.g., \citealp{koren2015fast,gonen2017fast,charles2018stability}).

Significant interest in the stability properties of iterative methods has arisen recently.
The work of \citet{hardt2016train} gave the first bounds on the uniform stability of stochastic gradient descent (SGD) for convex and smooth optimization, showing it grows linearly with the number of optimization steps.
Their result apply with minor modification to full-batch gradient descent (GD) as seen in following work \citep{feldman2018generalization,chen2018stability}.
Bounds for the stability of SGD and GD in the non-smooth case was studied by \citet{lei2020fine,bassily2020stability} and revealed a significant gap in stability between smooth and non-smooth optimization, indicating the importance of smoothness for stability.
Furthermore, algorithmic stability was instrumental in stochastic mini-batched iterative optimization (e.g.,
\citealp{wang2017memory,agarwalAKTZ20}), and has been pivotal to the
design and analysis of differentially private optimization algorithms
\citep{wu2017bolt,bassily2019private,feldman2020private}, both of which focused
mainly on smooth optimization.

%% file: sec_preliminaries.tex
\section{Preliminaries}
\subsection{Smooth convex optimization}
In this work we are interested in optimization of convex and smooth functions over a closed convex set $\xdomain \subseteq \R^d$.
A function $f$ is said to be $\sm$-smooth w.r.t.~a norm $\norm{\cdot}$ if its gradient is $\sm$-Lipschitz w.r.t.~$\norm{\cdot}$, namely $\dnorm{\nabla f(x)-\nabla f(y)} \leq \sm \norm{x-y}$ for all $x,y \in \xdomain$.
Here $\dnorm{\cdot}$ is the dual norm of $\norm{\cdot}$.
This smoothness condition also yields the following quadratic bound for all $x,y \in \xdomain$: $f(y) \leq f(x) + \nabla f(x)\cdot(y-x)+\frac{\sm}{2}\norm{y-x}^2$.
A function $f$ is said to be $\sc$-strongly convex w.r.t.~a norm $\norm{\cdot}$ if for all $x,y \in \xdomain$, we have $f(y) \geq f(x) + \nabla f(x) \cdot (y-x) + \frac{\sc}{2}\norm{y-x}^2$.

\subsection{Algorithmic stability}

In this work we consider the well known \textit{uniform stability} \citep{bousquet2002stability} in the following general setting of supervised learning.
There is an unknown distribution $\D$ over a sample set $\Z$ from which examples are drawn.
Given a training set $S=(z_1,\dots,z_n)$ of $n$ samples drawn i.i.d.~from $\D$, the objective is finding a model $x \in \xdomain$ with a small \emph{population risk}:
\begin{align*}
    \risk(x) 
    \eqdef 
    \E_{z \sim \D} [ \ell(x;z) ]
    ,
\end{align*}
where $\ell(x;z)$ is the loss of the model described by $x$ on an example $z$.
We cannot evaluate the population risk directly, thus, optimization will be applied on the \emph{empirical risk} with respect to the sample $S$, given by
\begin{align*}
    \emprisk(x)
    \eqdef 
    \frac{1}{n} \sum_{i=1}^{n} \ell(x;z_i)
    .
\end{align*}
We use the following notion of \textit{uniform stability}.%
\footnote{The definition is suitable for deterministic algorithms and is sufficient for the scope of this work. A more general definition exists for randomized algorithms ~\citep[e.g.,][]{hardt2016train,feldman2018generalization}.}
\begin{definition}[uniform stability] \label{def:uni_stab}
Algorithm $A$ is $\epsilon$-uniformly stable if for all $S, S' \in \Z^n$ such that $S, S'$ differ in at most one example, the corresponding outputs $A(S)$ and $A(S')$ satisfy
\begin{align*}
    \sup_{z \in \Z} ~ \abs{\ell(A(S);z)-\ell(A(S');z)} 
    \leq 
    \epsilon.
\end{align*}
\end{definition}
A known result of \citet{bousquet2002stability} is a bound on the expected generalization error of an $\epsilon$-uniformly stable algorithm $\alg$,
\begin{align*}
    \E_S[\risk(\alg(S))-\emprisk(\alg(S))] \leq \epsilon.
\end{align*}

\subsection{Relatively-smooth convex optimization}
In order to handle regularization over general norms we will need the framework of relatively smooth and convex functions.
Given a $1$-strongly convex function $\reg: \xdomain \mapsto \R$, the Bregman divergence of $\reg$ is defined as $\breg(y,x) \eqdef \reg(y)-\reg(x)-\nabla \reg(x) \cdot (y-x)$. With the definition of Bregman divergence we can define relative smooth and strong convexity.
\begin{definition}[relative strong convexity]
    A function $f: \xdomain \mapsto \R$ is $\sc$-strongly convex
    relative to $\reg: \xdomain \mapsto \R$ if for any $x,y \in \xdomain$,
      $  f(x) + \nabla f(x) \cdot (y-x) + \sc \breg(y,x) \leq f(y)$.
\end{definition}
\begin{definition}[relative smoothness]
    A function $f: \xdomain \mapsto \R$ is $\sm$-smooth
    relative to $\reg: \xdomain \mapsto \R$ if for any $x,y \in \xdomain$,
      $  f(x) + \nabla f(x) \cdot (y-x) + \sm \breg(y,x) \geq f(y)$.
\end{definition}
Our analysis relates to the so called Mirror Descent method which is used in both smooth and relatively-smooth convex optimization \citep{nemirovskij1983problem,lu2018relatively}.
The standard step of Mirror Descent over $\sm$-smooth (or relatively smooth) functions with a $1$-strongly convex regularization $\reg(\cdot)$ is as follows:
\begin{align} \label{eq:mirror_step}
    x_{t+1} 
    = 
    \argmin_{x \in \xdomain} \brk[c]!{ \nabla f(x_t) \cdot (x-x_t) + \sm \breg (x,x_t) }
    . 
\end{align}

%% file: arxiv/sec_l2_arxiv.tex
\section{General reduction in Euclidean geometry} \label{sec:l2}
In this section, we present a general method for converting a black-box smooth optimization algorithm with a given convergence guarantee to a uniformly stable algorithm with nearly the same convergence rate.
Let $\alg$ be an iterative algorithm for minimization of $\sm$-smooth convex functions with convergence rate of the form
\begin{align*}
    f(x_t)-f(\xstar) \leq \frac{C \sm \norm{x_0-\xstar}^2}{t^\gamma},
\end{align*}
where $\xstar \in \argmin_{x \in \xdomain} f(x)$, $x_t=\alg(f,\sm,x_0,t)$ is the output of the algorithm at step $t$ on a $\sm$-smooth function $f$ initialized at $x_0$, and constants $C,\gamma>0$.  Note that it is enough to consider $0 < \gamma \leq 2$, as $O(1/t^2)$ is the optimal convergence rate for smooth optimization.

\begin{algorithm2e}[h]
    \SetAlgoLined
    \DontPrintSemicolon
    \KwIn{Sample $S=\brk[b]{z_i}_{i=1}^n$, base algorithm $\alg$, initialization $x_0$, parameters $\sm$,$\lip$,$\diam$}
    \KwOut{A sequence of iterates $\brk[c]{x_t}$.}
    $\sc_0 \gets \frac{\sm}{4} ; y_0 \gets x_0 ; t_0 \gets 0.$\;
    \For{$k \gets 0,1,\dots$}{
        $\nepochs_k \gets \max \brk[c]!{1,2 \log_2\brk{\ifrac{ \sc_k D n}{G}}} ; \thalf{k} \gets \brk1{4 C \brk{1+\ifrac{\sm}{\sc_k}}}^{1/\gamma} ; t_{k+1} \gets t_k + \thalf{k} \nepochs_k$.\;
        $x_{t_k},x_{t_k+1},\dots,x_{t_{k+1} - 1} \gets y_{k} $.\;
        $y_{k,0} \gets y_{k}$.\;
        \For{$j \gets 0$ \KwTo $\nepochs_k-1$}{ 
            $y_{k,j+1} \gets \alg(\emprisk(x) + \frac{\sc_k}{2}\norm{x-x_0}^2,\sc_k+\sm,y_{k,j},\thalf{k})$.
        }
        $y_{k+1} \gets y_{k,\nepochs_k};$
        $\sc_{k+1} \gets \sc_{k}/2$.\;
    }
    \caption{$\stabregconvex(\alg)$: Uniformly Stable Optimization in L2} \label{alg:l2}
\end{algorithm2e}

Our general black-box conversion procedure, we name $\stabregconvex$, is presented in \cref{alg:l2}.
It is based on running the base algorithm $\alg$ in epochs, where in each epoch we optimize a regularized version of the empirical risk,  $\regemprisk{k}(x) = \emprisk(x) + \ifrac{\sc_k \norm{x-x_0}^2}{2}$.
To do so, we invoke $\alg$ for $\nepochs_k$ times, each time halving the distance to the regularized minimizer by performing $\thalf{k}$ steps.
Between epochs, we halve the regularization magnitude.

The basic idea behind \stabregconvex is simple: in each epoch the algorithm converges close to the minimizer of the regularized objective $\regemprisk{k}(x)$, which is a stable function of the sample thanks to the added strongly convex regularization. The convergence to the minimizer can be exponentially fast since $\regemprisk{k}(x)$ is strongly convex and smooth, thus the output of each epoch is also stable.
Following each epoch we decrease the regularization, thus converging closer and closer to the actual minimizer of the empirical (unregularized) risk while still maintaining uniform stability.

Note that if $\alg$ is a first-order algorithm that uses only gradient access to the empirical risk (such as GD or NAG), then $\stabregconvex(\alg)$ can also be implemented using first-order access, as we only add a simple L2 regularization term to the empirical risk whose gradient is easy to compute.

The main result of this section is the following theorem which establishes convergence and stability guarantees for $\stabregconvex$.

\begin{theorem} \label{thm:main_result1}
    Assume $\ell(\cdot,z)$ is convex, $\sm$-smooth and $\lip$-Lipschitz w.r.t.~$\norm{\cdot}_2$ on $\xdomain$.
    Let $\alg$ be an optimization algorithm with convergence rate $\ifrac{C \sm \norm{x_0-\xstar}^2}{t^\gamma}$.
    Then the iterates $\{x_t\}_t$ produced by $\stabregconvex(\alg)$ initialized at $x_0$ such that $\diam \geq \norm{x_0 - \xstar}$ satisfy the following, for all $t$:
    \begin{enumerate}[label=(\roman*)]
        \item
        $x_t$ is $O\brk{\ifrac{\lip^2 t^\gamma}{C \sm n } }$-uniformly stable.
        \item
        $
            \emprisk(x_t)-\emprisk(\xstar) 
            = 
            \otil\brk{\ifrac{C\sm D^2}{t^\gamma}}
            .
        $
    \end{enumerate}
\end{theorem}
Following are lemmas needed to prove \cref{thm:main_result1}.
In the lemmas we will refer to the values computed in \stabregconvex (such as $y_k$, $x_t$, etc.) for an arbitrary input. We will use the notation $\regemprisk{k}(x) = \emprisk(x) + \frac{\sc_k}{2}\norm{x-x_0}^2$ for the regularized objective we optimize at iteration $k$ of the algorithm, and $\xstar_k$ to be its minimizer.

We begin with a technical lemma that bounds the distance between the minimizers of two regularized objectives.

\begin{lemma} \label{claim:reg_minima_distance}
    Let $f(x)$ a convex function. Let $\xstar_1$ and $\xstar_2$ be the minimizers of $f(x)+\frac{\sc_1}{2}\norm{x-x_0}^2$ and $f(x)+\frac{\sc_2}{2}\norm{x-x_0}^2$ respectively, for some $x_0 \in \R^d$, $\sc_1 > 0$ and $\sc_2 \geq 0$. Then,
    \begin{align*}
        \norm{\xstar_1-\xstar_2}^2 \leq \frac{\sc_1-\sc_2}{\sc_1+\sc_2} \brk*{\norm{\xstar_2-x_0}^2-\norm{\xstar_1-x_0}^2}.
    \end{align*}
\end{lemma}
\begin{proof}%
    From the strong convexity of $f(x)+\frac{\sc_1}{2}\norm{x-x_0}^2$,
    \begin{align*}
        \frac{\sc_1}{2}\norm{\xstar_2-\xstar_1}^2
        &\leq f(\xstar_2)+\frac{\sc_1}{2}\norm{\xstar_2-x_0}^2-f(\xstar_1)-\frac{\sc_1}{2}\norm{\xstar_1-x_0}^2.
    \end{align*}
    Similarly,
    \begin{align*}
        \frac{\sc_2}{2}\norm{\xstar_1-\xstar_2}^2
        &\leq f(\xstar_1)+\frac{\sc_2}{2}\norm{\xstar_1-x_0}^2-f(\xstar_2)-\frac{\sc_2}{2}\norm{\xstar_2-x_0}^2 \\
        &\implies f(\xstar_2)-f(\xstar_1) \leq \frac{\sc_2}{2}\norm{\xstar_1-x_0}^2-\frac{\sc_2}{2}\norm{\xstar_2-x_0}^2 -\frac{\sc_2}{2}\norm{\xstar_1-\xstar_2}^2.
    \end{align*}
    Combining the two inequalities,
    \begin{align*}
        \frac{\sc_1}{2}\norm{\xstar_2-\xstar_1}^2 \leq \frac{\sc_1-\sc_2}{2}\brk*{\norm{\xstar_2-x_0}^2-\norm{\xstar_1-x_0}^2} - \frac{\sc_2}{2}\norm{\xstar_1-\xstar_2}^2,
    \end{align*}
    and we obtain the desired result by rearranging the terms.
\end{proof}
An immediate corollary from the lemma above for $\sc_2=0$ is the following.
\begin{corollary} \label{lemma:reg_distance}
    Let $f(x)$ be a convex function and let $f^{(\sc)}(x)=f(x)+\frac{\sc}{2}\norm{x-x_0}^2$ for some $\sc>0$ and $x_0$. Let $\xstar \in \argmin_x f(x)$ and $\xstar_\sc \in \argmin_x f^{(\sc)}(x)$. Then
    \begin{align*}
        \norm{x_0-\xstar_\sc}^2 + \norm{\xstar_\sc-\xstar}^2 \leq \norm{x_0-\xstar}^2.
    \end{align*}
\end{corollary}
Following lemma present the convergence in both function value and parameter distance of $y_{k+1}$ with respect to $\regemprisk{k}$.
\begin{lemma} \label{lemma:halving}
    For all $k \geq 0$, 
    \begin{align*}
        \regemprisk{k}(y_{k+1})-\regemprisk{k}(\xstar_k) \leq \frac{\sc_k \norm{y_{k}-\xstar_{k}}^2}{2^{\nepochs_k+1}} \qquad \text{ and } \qquad \norm{y_{k+1}-\xstar_{k}}^2 \leq \frac{\norm{y_{k}-\xstar_{k}}^2}{2^{\nepochs_k}}.
    \end{align*}
\end{lemma}
\begin{proof}
    From strong convexity and the rate of $\alg$,
    \begin{align*}
        \frac{\sc_k}{2}\norm{y_{k,j+1}-\xstar_k}^2
        &\leq \regemprisk{k}(y_{k,j+1})-\regemprisk{k}(\xstar_k) %
        \leq \frac{C \brk*{\sc_k+\sm} \norm{y_{k,j}-\xstar_k}^2}{4 C \brk*{1+\frac{\sm}{\sc_k}}} %
        = \frac{\sc_k \norm{y_{k,j}-\xstar_k}^2}{4}.
    \end{align*}
    Hence,
    \begin{align*}
        \norm{y_{k,j+1}-\xstar_k}^2 \leq \frac{\norm{y_{k,j}-\xstar_k}^2}{2}.
    \end{align*}
    Repeating this argument and substituting $y_k=y_{k,0}$,
    \begin{align*}
        \norm{y_{k,j+1}-\xstar_k}^2 \leq \frac{\norm{y_{k,0}-\xstar_k}^2}{2^{j+1}}
        =\frac{\norm{y_{k}-\xstar_k}^2}{2^{j+1}}.
    \end{align*}
    Thus,
    \begin{align*}
        \regemprisk{k}(y_{k,\nepochs_k})-\regemprisk{k}(\xstar_k)
        &\leq \frac{C(\sc_k+\sm) \norm{y_{k,\nepochs_k-1}-\xstar_k}^2}{4 C \brk*{1+\sm/\sc_k}} %
        = \frac{\sc_k \norm{y_{k,\nepochs_k-1}-\xstar_k}^2}{4} %
        \leq \frac{\sc_k \norm{y_{k}-\xstar_k}^2}{2^{\nepochs_k+1}}.
    \end{align*}
    We conclude by substituting $y_{k+1}=y_{k,\nepochs_k}$.
\end{proof}
Each time we update the regularization we start from the last $y_k$ iteration. The following lemma bounds the distance of the last iteration to the new minimizer.
\begin{lemma} \label{lemma:bounded_diameter}
    For all $k \geq 0$ it holds that $\norm{y_{k}-\xstar_{k}} \leq \norm{x_0-\xstar_{k}}$.
\end{lemma}
\begin{proof}%
    For $k=0$ the claim is immediate as $y_0=x_0$.
    For $k>0$, using the triangle inequality,
    \begin{align*}
        \norm{y_k-\xstar_k}
        &\leq \norm{y_k-\xstar_{k-1}} + \norm{\xstar_{k-1}-\xstar_k}.
    \end{align*}
    By \cref{claim:reg_minima_distance},
    \begin{align*}
        \norm{\xstar_{k-1}-\xstar_k}^2
        &\leq \frac{\sc_{k-1}-\sc_{k}}{\sc_{k-1}+\sc_{k}} \brk*{\norm{\xstar_{k}-x_0}^2-\norm{\xstar_{k-1}-x_0}^2}
        \leq \frac{\norm{\xstar_{k}-x_0}^2-\norm{\xstar_{k-1}-x_0}^2}{3}.
    \end{align*}
    Due to \cref{lemma:halving},
    \begin{align*}
        \norm{y_k-\xstar_{k-1}}
        &\leq \frac{\norm{y_{k-1}-\xstar_{k-1}}}{2^{\nepochs_k/2}} \\
        &\leq \frac{\norm{x_0-\xstar_{k-1}}}{2^{\nepochs_k/2}} \tag{induction} \\
        &\leq \frac{\norm{x_0-\xstar_{k-1}}}{\sqrt{2}}. \tag{$\nepochs_k \geq 1$}
    \end{align*}
    Thus,
    \begin{align*}
        \norm{y_k-\xstar_k}
        & \leq \frac{\norm{x_0-\xstar_{k-1}}}{\sqrt{2}} + \sqrt{\frac{\norm{\xstar_{k}-x_0}^2-\norm{\xstar_{k-1}-x_0}^2}{3}} \\
        &= \norm{\xstar_{k}-x_0} \brk*{\frac{\norm{x_0-\xstar_{k-1}}/\norm{\xstar_k-x_0}}{\sqrt{2}} + \sqrt{\frac{1-\norm{\xstar_{k-1}-x_0}^2/\norm{\xstar_k-x_0}^2}{3}}}.
    \end{align*}
    From \cref{claim:reg_minima_distance}, $\norm{\xstar_{k-1}-x_0} \leq \norm{\xstar_k-x_0}$, and since for $0 \leq x \leq 1$, $\frac{x}{\sqrt{2}}+\sqrt{\frac{1-x^2}{3}} \leq 1$,
    \begin{align*}
        \norm{y_k-\xstar_k}
        &\leq \norm{\xstar_{k}-x_0}. \qedhere
    \end{align*}
\end{proof}
Next lemma yields the convergence guarantee of the $y_k$ sequence.
\begin{lemma} \label{lemma:yk-convergence}
    For all $k \geq 0$,
    \begin{align*}
        \emprisk(y_{k+1})-\emprisk(\xstar)
        &\leq \frac{3 \sc_k \norm{x_0-\xstar}^2}{4}.
    \end{align*}
\end{lemma}
\begin{proof}[of \cref{lemma:yk-convergence}]
    By the strong convexity of $\regemprisk{k}$,
    \begin{align*}
        \regemprisk{k}(y_{k+1})-\regemprisk{k}(\xstar)
        &= \regemprisk{k}(y_{k+1})-\regemprisk{k}(\xstar_k)+\regemprisk{k}(\xstar_k)-\regemprisk{k}(\xstar) \\
        &\leq \regemprisk{k}(y_{k+1})-\regemprisk{k}(\xstar_k) - \frac{\sc_k}{2} \norm{\xstar_k-\xstar}^2.
    \end{align*}
    Hence, using the definition of $\regemprisk{k}$,
    \begin{align*}
        \emprisk(y_{k+1})-\emprisk(\xstar)
        &\leq \regemprisk{k}(y_{k+1})-\regemprisk{k}(\xstar) + \frac{\sc_k}{2}\norm{\xstar-x_0}^2 \\
        &\leq  \regemprisk{k}(y_{k+1})-\regemprisk{k}(\xstar_k) + \frac{\sc_k}{2}\brk*{\norm{\xstar-x_0}^2-\norm{\xstar_k-\xstar}^2} \\
        &\leq \frac{\sc_k \norm{y_k-\xstar_k}}{4} + \frac{\sc_k \norm{\xstar-x_0}^2}{2}. \tag{\cref{lemma:halving} and $\nepochs_k \geq 1$}
    \end{align*}
    By \cref{lemma:bounded_diameter},
    \begin{align*}
        \emprisk(y_{k+1})-\emprisk(\xstar)
        &\leq \frac{\sc_k}{4}\brk*{\norm{x_0-\xstar_k}^2 + 2 \norm{\xstar-x_0}^2}.
    \end{align*}
    We conclude by applying \cref{lemma:reg_distance}.
\end{proof}
The following standard lemma bounds the distance between the minimizers of two functions where one of them is strongly convex. (For completeness, we include a proof in~\cref{proof:lemma:reg_distance_bounded}.)
\begin{lemma} \label{lemma:reg_distance_bounded}
    Let $f_1,f_2 : \xdomain \mapsto \R$ be convex and $\sc$-strongly convex functions (respectively) defined over a closed and convex domain $\xdomain \subseteq \R^d$, and let $x_1 \in \argmin_{x \in \xdomain} f_1(x)$ and $x_2 \in \argmin_{x \in \xdomain} f_2(x)$.
    Then for $h = f_2 - f_1$ we have
    \begin{align*}
        \norm{x_2 - x_1} &\leq \frac{2}{\sc}\dnorm{\nabla h(x_1)}.
    \end{align*}
\end{lemma}
Following is the stability guarantee for the $y_k$ iterates of $\stabregconvex$.
\begin{lemma} \label{lemma:yk-stability}
    For $k \geq 0$, the iterate $y_{k+1}$ produced by $\stabregconvex$ is $(\ifrac{6 \lip^2}{n \sc_k})$-uniformly stable.
\end{lemma}
\begin{proof}[of \cref{lemma:yk-stability}]
    Let $S=\brk[b]{z_1,\dots,z_n}$ and $S'=\brk[b]{z_1,\dots,z_{i-1},z'_i,z_{i+1},\dots,z_n}$ be two neighboring datasets and
    let $\brk[c]{y_k}_k,\brk[c]{\yt_k}_k$ be the two ``$y_k$'' iterates obtained from $\stabregconvex$ respectively.
    Using the triangle inequality,
    \begin{align*}
        \norm{y_{k+1}-\yt_{k+1}}
        &\leq \norm{y_{k+1}-\xstar_k} + \norm{\xstar_k-\xtstar_k} + \norm{\xtstar_k-\yt_{k+1}}.
    \end{align*}
    Using \cref{lemma:halving} and the definition of $\nepochs_k$,
    \begin{align*}
        \norm{y_{k+1}-\xstar_k}
        &\leq \frac{\norm{y_{k}-\xstar_k}}{2^{\nepochs_k/2}}
        \leq \frac{G \norm{y_{k}-\xstar_k}}{D n \sc_k}.
    \end{align*}
    Combining \cref{lemma:bounded_diameter} and \cref{lemma:reg_distance},
    \begin{align*}
        \norm{y_{k}-\xstar_k}
        &\leq \norm{\xstar_{k}-x_0}
        \leq \norm{\xstar-x_0}
        \leq D.
    \end{align*}
    Thus
    $
        \norm{y_{k+1}-\xstar_k} \leq \frac{G}{n \sc_k},
    $
    and similarly $\norm{\xtstar_k-\yt_{k+1}} \leq \frac{G}{n \sc_k}$. Hence,
    \begin{align*}
        \norm{y_{k+1}-\yt_{k+1}} \leq \norm{\xstar_k-\xtstar_k} + \frac{2G}{n \sc_k}.
    \end{align*}
    We will now focus on bounding $\norm{\xstar_k-\xtstar_k}$.
    By invoking \cref{lemma:reg_distance_bounded} with $\regemprisk[S]{k}$ and $\regemprisk[S']{k}$,
    \begin{align*}
        \norm{\xstar_k - \xtstar_k}
        &\leq \frac{2 \dnorm{\ell(\xstar_k;z_i')-\ell(\xstar_k;z_i)}}{n \sc_k}
        \leq \frac{4 \lip}{n \sc_k},
    \end{align*}
    where we have used the fact that $\ell(\cdot,z_i)$ and $\ell(\cdot,z_i')$ are $\lip$-Lipschitz.
    Thus,
    $
        \norm{y_{k+1}-\yt_{k+1}}
        \leq \frac{6G}{n \sc_k}.
    $
    Again using the fact that $\ell(\cdot,z)$ is Lipschitz,
    \begin{align*}
        \sup_{z \in \Z} \abs{\ell(y_{k+1};z)-\ell(\yt_{k+1};z)}
        &\leq \lip\norm{y_{k+1}-\yt_{k+1}}
        \leq \frac{6 \lip^2}{n \sc_k},
    \end{align*}
    hence we establish uniform stability.
\end{proof}
We are now ready to prove \cref{thm:main_result1}.
\begin{proof}[of \cref{thm:main_result1}]
If $0<t<t_1$ then $y_t=x_0$ and uniform stability is immediate.
Regarding convergence, by smoothness,
\begin{align*}
    \emprisk(x_0)-\emprisk(\xstar)
    &\leq \frac{\sm \diam^2}{2}
    \leq \frac{\sm \diam^2 (\thalf{0}\nepochs_0)^{\gamma}}{2 t^\gamma}
    = \otil\brk*{\frac{C \sm \diam^2}{t^\gamma}}.
\end{align*}
Let $K=\max \brk[c]{k:t_{k+1} \leq t}$ for some $t>0$ which implies $x_t=y_{K+1}$.
To obtain uniform stability by \cref{lemma:yk-stability} we need to lower bound $\sc_{K}$. Thus,
\begin{align*}
    t
    &\geq t_{K+1}
    \geq \thalf{K}\nepochs_{K}
    \geq \brk*{\frac{4 C \sm}{\sc_{K}}}^{1/\gamma}
    \implies \sc_k = \Omega\brk*{\frac{C \sm}{t^\gamma}}.
\end{align*}
Thus, $x_t$ is $O\brk*{\ifrac{\lip^2 t^\gamma}{C \sm n}}$-uniformly stable.
For the convergence result we need to upper bound $\sc_K$ and invoke \cref{lemma:yk-convergence}.
First we bound $\thalf{k}$ by a geometric series,
\begin{align*}
    \thalf{k}
    &= \brk3{4 C \brk3{\frac{\sm}{\sc_k} + 1}}^{1/\gamma}
    = \brk*{4 C \brk1{2^{k+2} + 1}}^{1/\gamma}
    \leq \brk1{20 C \cdot 2^{k}}^{1/\gamma}
    \implies \sum_{i=0}^{k} \thalf{i}
    \leq (20 C)^{1/\gamma} \frac{2^{(k+1)/\gamma}}{2^{1/\gamma}-1}.
\end{align*}
Thus, using the definition of $\nepochs_k$,
\begin{align*}
    t
    &< t_{K+2}
    = \sum_{k=0}^{K+1} \nepochs_k \thalf{k}
    \leq \max\brk[c]*{1,2 \log_2 \frac{\sm D n}{4G}} \sum_{k=0}^{K+1} \thalf{k}
    \leq \brk*{20 C}^{1/\gamma} \max\brk[c]*{1,2 \log_2 \frac{\sm D n}{4G}} \frac{2^{(K+2)/\gamma}}{2^{1/\gamma}-1}
    .
\end{align*}
Rearranging the terms and using $0 < \gamma \leq 2$,
\begin{align*}
    \frac{1}{2^{K+2}}
    &\leq 
    \frac{10 C}{t^\gamma \brk*{1-2^{-1/\gamma}}^\gamma}\max\brk[c]*{1,2 \log_2^\gamma \frac{\sm D n}{4G}}
    =O\brk*{\frac{C}{t^\gamma}\max\brk[c]*{1,2 \log_2^\gamma \frac{\sm D n}{4G}}}.
\end{align*}
We conclude using \cref{lemma:yk-convergence} with $x_t=y_{K+1}$ and $\sc_K=\sm/2^{K+2}$,
\begin{align*}
    \emprisk(x_t)-\emprisk(\xstar) 
    &\leq \frac{3 \sc_K \diam^2}{4}
    = O\brk2{\frac{C\sm D^2}{t^\gamma} \max\brk[c]2{1,\log^{\gamma}\brk2{\frac{\sm D n}{G}}}}
    . \qedhere
\end{align*}
\end{proof}

%% file: arxiv/sec_md_arxiv.tex
\section{Stable Mirror Descent for general norms} 
\label{sec:mirror_alternative}

In this section we provide a uniformly stable variant of Mirror Descent for empirical risk minimization in general normed spaces.
The algorithm, which we term $\stabregrel$ is presented in \cref{alg:stable_md}).
Here we assume that the loss function $\ell: \xdomain \times \Z \mapsto \R$ is convex, $\sm$-smooth and $\lip$-Lipschitz (in its first argument) w.r.t.~a general $\norm{\cdot}$.

Like the standard Mirror Descent, the algorithm is parameterized by a regularization function $\reg(x)$ which is $1$-strongly convex w.r.t.~$\norm{\cdot}$.
Let $x_0 = \argmin_{x \in \xdomain} \reg(x)$, $\xstar \in \argmin_{x \in \xdomain} \emprisk(x)$ and let $\diam^2 \geq \reg(\xstar)-\reg(x_0)$.
\begin{algorithm2e}[h]
    \SetAlgoLined
    \DontPrintSemicolon
    \KwIn{Sample $S=\brk[b]{z_i}_{i=1}^n$, smoothness $\sm$, $x_0$, regularization parameter $\sc$, number of steps $T$.}
    \For{$t \gets 0$ \KwTo $T-1$}{
        \[x_{t+1} \gets \argmin_{x \in \xdomain} \brk[c]{\nabla \emprisk(x_t) \cdot (x-x_t)+\sm\breg(x,x_t)+\sc \reg(x)}.\]
    }
    \KwOut{$x_T$.}
    \caption{\stabregrel: Uniformly Stable Mirror Descent} \label{alg:stable_md}
\end{algorithm2e}
In each step of $\stabregrel$ we perform linearization of $\emprisk(x_t)$ and use $\sc \reg(x)$ as a regularization in addition to the Bregman of the mirror descent step, which lets us obtain linear rate convergence on the regularized objective.
Hence, we obtain stability by converging near to a regularized minimizer at a minimal cost.
$\sc$ is carefully tuned to balance the stability and empirical error. 
Note that the update step of $\stabregrel$ is in fact a Mirror Descent step on the regularized function $\regemprisk{\sc}(x) \eqdef \emprisk(x) + \sc \reg(x)$ with smoothness of $\sc+\sm$. This can be seen by comparing the $\argmin$ step of $\stabregrel$ and that of Mirror Descent (\cref{eq:mirror_step}), and is written formally in the proof of \cref{thm:main_result2}.

Further, note that first-order access to the empirical risk suffices for implementing our method as we only access it by performing linearization in each step.
Following is the main result for this section, describing the stability and convergence of \cref{alg:stable_md}.
\begin{theorem} \label{thm:main_result2}
    Assume $\ell(\cdot,z)$ is convex, $\sm$-smooth and $\lip$-Lipschitz w.r.t.~$\norm{\cdot}$ on $\xdomain$, $\reg(x)$ is $1$-strongly convex w.r.t.~$\norm{\cdot}$ on $\xdomain$, $x_0 = \argmin_{x \in \xdomain} \reg(x)$ and $\diam^2 \geq \reg(\xstar)-\reg(x_0)$. Then given $T \geq 2 \log \frac{\sm \diam n}{\lip}$, the output of $\stabregrel$ (the final iterate $x_T$) with $\sc=\frac{\sm}{T} \max\brk[c]!{1,2\log_2 \frac{\sm \diam n}{\lip T}}$ satisfies the following:
    \begin{itemize}
        \item
        $x_T$ is $O\brk{\ifrac{G^2 T}{\sm n}}$-uniformly stable.
        \item
        $\emprisk(x_T)-\emprisk(\xstar) = \otil\brk{\ifrac{\sm \diam^2}{T}}$.
    \end{itemize}
\end{theorem}

The following lemmas are used in order to prove \cref{thm:main_result2}.
As we mentioned, adding regularization can impair smoothness (cf.~\cref{sec:non_smooth_regulator}), hence we cannot appeal directly to classical bounds for smooth Mirror Descent.
The next lemma show that adding regularization, although not necessarily smooth, is indeed relatively smooth and strongly convex.
\begin{lemma}\label{lemma:reg-rel-sc-sm}
    Let $f(x)$ be a convex and $\sm$-smooth function w.r.t.~$\norm{\cdot}$. Let $R(x)$ be $1$-strongly convex w.r.t.~$\norm{\cdot}$. Then $f^{(\sc)}(x) \eqdef f(x)+\sc \reg(x)$ for $\sc>0$ is $(\sc+\sm)$-smooth and $\sc$-strongly convex relative to $\reg(x)$.
\end{lemma}
\begin{proof}
Due to convexity of $f$,
$f(y) \geq f(x) + \nabla f(x) \cdot (y-x).$
Since $\breg(y,x)=\reg(y)-\reg(x)-\nabla \reg(x) \cdot (y-x)$,
\begin{align*}
    f(y) + \sc \reg(y) \geq f(x) + (\nabla f(x)+\sc \nabla \reg(x)) \cdot (y-x) + \sc \reg(x) + \sc \breg(y,x).
\end{align*}
Hence, by the definition of $f^{(\sc)}(x)$,
\begin{align*}
    f^{(\sc)}(y) \geq f^{(\sc)}(x) + \nabla f^{(\sc)}(x) \cdot (y-x) + \sc \breg(y,x),
\end{align*}
and we conclude that $f^{(\sc)}(x)$ is $\sc$-strongly convex relative to $\reg(x)$.
Since $f(x)$ is $\sm$-smooth,
\begin{align*}
    f(y) \leq f(x) + \nabla f(x) \cdot (y-x) + \frac{\sm}{2}\norm{y-x}^2.
\end{align*}
Using the inequality $\breg(y,x) \geq \frac{1}{2}\norm{y-x}^2$ (since $\reg$ is $1$-strongly convex),
\begin{align*}
    f(y) \leq f(x) + \nabla f(x) \cdot (y-x) + \sm \breg(y,x).
\end{align*}
Adding $\sc \breg(y,x)$ to both sides and using the definition of $\breg(y,x)$,
\begin{align*}
    f(y) + \sc \reg(y) - \sc \reg(x) - \sc \nabla \reg(x) \cdot (y-x)
    &\leq f(x) + \nabla f(x) \cdot (y-x) + (\sc+\sm) \breg(y,x).
\end{align*}
Hence, by the definition of $f^{(\sc)}(x)$,
\begin{align*}
    f^{(\sc)}(y)
    &\leq f^{(\sc)}(x) + \nabla f^{(\sc)}(x) \cdot (y-x) + (\sc+\sm) \breg(y,x),
\end{align*}
and we conclude that $f^{(\sc)}(x)$ is $(\sc+\sm)$-smooth relative to $\reg(x)$.
\end{proof}
Hence, one can take advantage of the method of \citet{lu2018relatively} for relatively-smooth convex optimization.
The following lemma is derived from the analysis of \citet{lu2018relatively};
we include a proof in \cref{proof:lemma:rel_mirror} for completeness.
\begin{lemma} \label{lemma:rel_mirror}
    Let $f(x)$ be $\sm$-smooth and $\sc$-strongly convex relative to $\reg(x)$. Then the sequence $\brk[c]{x_t}$ defined by \cref{eq:mirror_step} satisfy:
    \begin{enumerate}[label=(\roman*)]
        \item
        $\brk[c]{f(x_t)}_t$ is monotonically decreasing.
        \item
        Let $\xstar \in \argmin_{x \in \xdomain} f(x)$. For all $t \geq 1$,
        $
            \breg(\xstar,x_t)
            \leq \brk!{1-\frac{\sc}{\sm}}^t \breg(\xstar,x_0).
        $
        \item
        For all $t \geq 1$ and $x \in \xdomain$,
        $
            f(x_t) - f(x)
            \leq \frac{\sc \breg(x,x_0)}{\brk{1+\frac{\sc}{\sm-\sc}}^t-1}.
        $
    \end{enumerate}
\end{lemma}
We are now ready to prove \cref{thm:main_result2}.
\begin{proof}[of \cref{thm:main_result2}]
    We start with showing that $\breg(\xstar_\sc,0) \leq \diam^2$. This inequality will be used for both stability and convergence results.
    By the zero-order optimality of $\xstar_\sc$ and $\xstar$,
    \begin{align*}
        &\emprisk(\xstar_\sc) + \sc \reg(\xstar_\sc)
        \leq \emprisk(\xstar) + \sc \reg(\xstar) \\
        \implies&
        \reg(\xstar) - \reg(\xstar_\sc)
        \geq \frac{\emprisk(\xstar_\sc) - \emprisk(\xstar)}{\sc}
        \geq 0.
    \end{align*}
    From the first-order optimality of $x_0$ which implies $\nabla \reg(x_0) \cdot(x_0-\xstar_\sc) \leq 0$,
    \begin{align*}
        \breg(\xstar_\sc,x_0)
        &= \reg(\xstar_\sc)-\reg(x_0) - \nabla \reg(x_0)\cdot(\xstar_\sc-x_0)
        \leq \reg(\xstar_\sc)-\reg(x_0).
    \end{align*}
    Combining the two inequalities,
    \begin{align} \label{eq:bound_breg_xstar_sc}
        \breg(\xstar_\sc,x_0)
        &\leq \reg(\xstar_\sc)-\reg(x_0)
        \leq \reg(\xstar)-\reg(x_0)
        \leq \diam^2.
    \end{align}
    Secondly we will show that our method in fact performs mirror steps on $\regemprisk{\sc}(x)=\emprisk(x)+\sc \reg(x)$ which is $(\sc+\sm)$-smooth and $\sc$-strongly convex relative to $\reg(x)$ by \cref{lemma:reg-rel-sc-sm}.
     The update step of $\stabregrel$ is
    \begin{align*}
        x_{t+1} = \argmin_{x \in \xdomain} \brk[c]{\nabla \emprisk(x_t) \cdot (x-x_t)+\sm\breg(x,x_t)+\sc \reg(x)}.
    \end{align*}
    Using the definition of $\breg(x,x_t)$,
    \begin{align*}
        &\nabla \emprisk(x_t) \cdot (x-x_t) + \sm \breg(x,x_t) + \sc \reg(x) \\
        &= \nabla \emprisk(x_t) \cdot (x-x_t) + (\sc+\sm) \breg(x,x_t) + \sc \reg(x) - \sc (\reg(x)-\reg(x_t)-\nabla \reg(x_t)\cdot(x-x_t)) \\
        &= (\nabla \emprisk(x_t)+\sc \nabla \reg(x_t)) \cdot (x-x_t) + (\sc+\sm) \breg(x,x_t) + \sc \reg(x_t).
    \end{align*}
    Thus,
    \begin{align*}
        &
        \argmin_{x \in \xdomain} \brk[c]1{\nabla \emprisk(x_t) \cdot (x-x_t) + \sm \breg(x,x_t) + \sc \reg(x)}
        \\
        &= \argmin_{x \in \xdomain} \brk[c]1{\nabla \regemprisk{\sc}(x_t) \cdot (x-x_t) + (\sc+\sm) \breg(x,x_t)}
        ,
    \end{align*}
    which is a mirror descent step (\cref{eq:mirror_step}) for $\regemprisk{\sc}$ with a smoothness of $\sm+\sc$.
    Hence, we can invoke \cref{lemma:rel_mirror}.
    
    Next follows the stability argument.
    Let $S=\brk[b]{z_1,\dots,z_n}$ and $S'=\brk[b]{z_1,\dots,z_{i-1},z'_i,z_{i+1},\dots,z_n}$.
    Let $x_T$ and $\xt_T$ be the outputs of \stabregrel on $S$ and $S'$ respectively. Let $\xstar_\sc = \argmin_{x \in \xdomain} \regemprisk{\sc}(x)$ and $\xtstar_\sc = \argmin_{x \in \xdomain} \regemprisk[S']{\sc}(x)$.
    Using the triangle inequality,
    \begin{align*}
        \norm{x_T-\xt_T}
        &\leq \norm{x_T - \xstar_\sc} + \norm{\xstar_\sc-\xtstar_\sc} + \norm{\xtstar_\sc-\xt_T}.
    \end{align*}
    By \cref{lemma:rel_mirror},
    \begin{align*}
        \breg(\xstar_\sc,x_T)
        &\leq \brk2{1-\frac{\sc}{\sc+\sm}}^T \breg(\xstar_\sc,x_0)
        = \brk2{\frac{1}{1+\frac{\sc}{\sm}}}^T \breg(\xstar_\sc,x_0).
    \end{align*}
    Using the inequality $\brk*{1+\frac{1}{x}}^{x} \geq 2$ for $x \geq 1$, with $\frac{\sm}{\sc} \geq 1$, $\brk!{1+\frac{\sc}{\sm}}^{\sm/\sc} \geq 2$.
    Note that $\frac{\sm}{\sc} \geq 1$ due to our assumption that $2 \log_2 \frac{\sm \diam n}{\lip} \leq T$.
    Thus,
    \begin{align*}
        \breg(\xstar_\sc,x_T)
        &\leq \brk2{\frac{1}{1+\frac{\sc}{\sm}}}^T \breg(\xstar_\sc,x_0)
        \leq 2^{-\frac{\sc T}{\sm}} \breg(\xstar_\sc,x_0).
    \end{align*}
    By $\sc \geq \frac{2 \sm}{T} \log_2 \frac{\sm \diam n}{\lip T}$ and \cref{eq:bound_breg_xstar_sc},
    \begin{align*}
        \breg(\xstar_\sc,x_T)
        &\leq \frac{\lip^2 T^2}{\sm^2 \diam^2 n^2}\breg(\xstar_\sc,x_0)
        \leq \frac{\lip^2 T^2}{\sm^2 n^2}. 
    \end{align*}
    From the strong convexity of $\reg(x)$, $\tfrac{1}{2} \norm{\xstar_\sc-x_T}^2 \leq \breg(\xstar_\sc,x_T)$.
    Hence,
    \begin{align*}
        \norm{\xstar_\sc-x_T}
        &\leq \frac{\sqrt{2} \lip T}{\sm n},
    \end{align*}
    and similarly $\norm{\xt_T-\xtstar_\sc} \leq \frac{\sqrt{2} \lip T}{\sm n}$.
    Now we will bound $\norm{\xstar_\sc-\xtstar_\sc}$.
    Using \cref{lemma:reg_distance_bounded} with $f_1=\regemprisk{\sc}$ ($\sc$-strongly convex since $\reg(x)$ is $1$-strongly convex) and $f_2=\regemprisk[S']{\sc}$,
    \begin{align*}
        \norm{\xstar_\sc - \xtstar_\sc}
        &\leq \frac{2 \dnorm{\ell(\xstar_\sc;z_i')-\ell(\xstar_\sc;z_i)}}{n \sc}
        \leq \frac{4 \lip}{n \sc},
    \end{align*}
    where we have used the fact that $\ell(\cdot,z_i)$ and $\ell(\cdot,z_i')$ are $\lip$-Lipschitz.
    Thus, since $\sc \geq \sm/T$,
    \begin{align*}
        \norm{x_{T}-\xt_{T}}
        &\leq \frac{4 \lip}{n \sc} + \frac{2\sqrt{2}\lip T}{n \sm}
        \leq \frac{(4+2\sqrt{2})\lip T}{n \sm}.
    \end{align*}
    Since $\ell(\cdot,z)$ is $\lip$-Lipschitz, we upper bound the uniform stability,
    \begin{align*}
        \sup_{z \in \Z} \abs{\ell(x_T;z)-\ell(\xt_T;z)}
        &\leq \lip \norm{x_{T}-\xt_{T}}
        \leq \frac{(4+2\sqrt{2})\lip^2 T}{n \sm}
        = O\brk*{\frac{G^2 T}{\sm n}}.
    \end{align*}
    We move on to the convergence of $x_T$.
    \begin{align*}
        \emprisk(x_T)-\emprisk(\xstar)
        &= \regemprisk{\sc}(x_T)-\regemprisk{\sc}(\xstar) + \sc (\reg(\xstar)-\reg(x_T)) \tag{$\regemprisk{\sc}(x) = \emprisk(x)+\sc \reg(x)$} \\
        &\leq \regemprisk{\sc}(x_T)-\regemprisk{\sc}(\xstar_\sc) + \sc (\reg(\xstar)-\reg(x_T)). \tag{$\xstar_\sc = \argmin_{x \in \xdomain} \regemprisk{\sc}(x)$}
    \end{align*}
    Again by \cref{lemma:rel_mirror}, and $T \geq \frac{\sm}{\sc}$ which implies $\brk1{1+\tfrac{\sc}{\sm}}^T \geq 2$,
    \begin{align*}
        \regemprisk{\sc}(x_T)-\regemprisk{\sc}(\xstar_\sc)
        &\leq \frac{\sc \breg(\xstar_\sc,x_0)}{\brk*{1+\frac{\sc}{\sm}}^T-1}
        \leq \sc \breg(\xstar_\sc,x_0).
    \end{align*}
    Thus, using the minimality of $x_0$,
    \begin{align*}
        \emprisk(x_T)-\emprisk(\xstar)
        &\leq \sc (\breg(\xstar_\sc,x_0) + \reg(\xstar)-\reg(x_T))
        \leq \sc (\breg(\xstar_\sc,x_0) + \reg(\xstar)-\reg(x_0)),
    \end{align*}
    and by \cref{eq:bound_breg_xstar_sc} and the definition of $\sc$,
    \begin{align*}
        \emprisk(x_T)-\emprisk(\xstar)
        &\leq 2 \sc \diam^2
        = \otil\brk*{\frac{\sm \diam^2}{T}}. \qedhere
    \end{align*}
\end{proof}

%% file: arxiv/appendix_arxiv.tex
\section{Missing Proofs}

\subsection{Proof of \cref{lemma:reg_distance_bounded}} \label{proof:lemma:reg_distance_bounded}
\begin{proof}
    The strong convexity of $f_2$ and $x_2$ being the minimum of $f_2$ implies
    \begin{align*}
        \nabla f_2(x_1)\cdot (x_1-x_2)
        &\geq f_2(x_1) - f_2(x_2) + \frac{\sc}{2} \norm{x_2-x_1}^2
        \geq \frac{\sc}{2} \norm{x_2-x_1}^2.
    \end{align*}
    From first-order optimality of $x_1$, $\nabla f_1(x_1) \cdot (x_1-x_2) \leq 0$. Thus,
    \begin{align*}
        \nabla f_2(x_1) \cdot (x_1-x_2)
        &= \nabla f_1(x_1)\cdot(x_1-x_2) + \nabla h(x_1)\cdot(x_1-x_2)
        \leq \nabla h(x_1)\cdot(x_1-x_2).
    \end{align*}
    Putting the two inequalities, together with H\"older inequality, yields
    \begin{align*}
        \frac{\sc}{2} \norm{x_2-x_1}^2
        &\leq \nabla h(x_1)\cdot(x_1-x_2)
        \leq \dnorm{\nabla h(x_1)} \norm{x_1-x_2}.
    \end{align*}
    Thus, $\norm{x_2 - x_1} \leq \frac{2}{\sc}\dnorm{\nabla h(x_1)}$.
\end{proof}

\subsection{Proof of \cref{lemma:rel_mirror}} \label{proof:lemma:rel_mirror}
The lemma is based on the analysis of \citet{lu2018relatively} (Theorem 3.1) which gives the convergence in terms of $f(x_t)-f(x)$. For \cref{thm:main_result2} we also need convergence in terms of $\frac{\breg(\xstar,x_t)}{\breg(\xstar,x_0)}$ for $\xstar \in \argmin_{x \in \xdomain} f(x)$. Thus, we repeat the argument for completeness.
The proof relies on the following Three-Point Property:
\begin{lemma}[Three-Point Property of \citet{tseng2008accelerated}] \label{lemma:tseng}
    Let $\phi(x)$ be a convex function, and let $\breg(\cdot,\cdot)$ be the Bregman distance for $\reg(\cdot)$. For a given vector $z$, let
    \begin{align*}
        z^+ \eqdef \argmin_{x \in \xdomain} \brk[c]{\phi(x)+\breg(x,z)}.
    \end{align*}
    Then
    \begin{align*}
        \phi(x)+\breg(x,z) \geq \phi(z^+) + \breg(z^+,z) + \breg(x,z^+) \text{ for all $x \in \xdomain$}.
    \end{align*}
\end{lemma}
\begin{proof}[of \cref{lemma:rel_mirror}]
    For any $x \in \xdomain$ and $t \geq 1$, from relative smoothness,
    \begin{align*}
        f(x_t)
        &\leq f(x_{t-1}) + \nabla f(x_{t-1}) \cdot (x_{t}-x_{t-1}) + \sm \breg(x_t,x_{t-1}).
    \end{align*}
    Using \cref{lemma:tseng} with $\phi(x)=\frac{1}{\sm}\nabla f(x_{t-1})\cdot (x-x_{t-1})$ and $z=x_{t-1}$,
    \begin{align*}
        \frac{1}{\sm}\nabla f(x_{t-1}) \cdot (x-x_{t}) + \breg(x,x_{t-1}) \geq \breg(x_t,x_{t-1}) + \breg(x,x_t).
    \end{align*}
    Hence,
    \begin{align*}
        f(x_t)
        &\leq f(x_{t-1}) + \nabla f(x_{t-1}) \cdot (x-x_{t-1}) + \sm \breg(x,x_{t-1}) - \sm \breg(x,x_{t}).
    \end{align*}
    Thus, from relative strong convexity,
    \begin{align*}
        f(x_t)
        &\leq f(x) + (\sm-\sc) \breg(x,x_{t-1}) - \sm \breg(x,x_{t}).
    \end{align*}
    Note that if $x=x_{t-1}$, $f(x_t) \leq f(x_{t-1})$. Hence, $\brk[c]*{f(x_{t})}_t$ is monotonically decreasing.
    From the definition of $\xstar$, $f(x_t) \geq f(\xstar)$. Hence, for $x=\xstar$ we obtain
    \begin{align*}
        \breg(\xstar,x_t)
        &\leq \frac{\sm-\sc}{\sm} \breg(\xstar,x_{t-1})
        = \brk*{1-\frac{\sc}{\sm}} \breg(\xstar,x_{t-1}).
    \end{align*}
    Repeating this argument and we obtain
    \begin{align*}
        \breg(\xstar,x_t) \leq \brk*{1-\frac{\sc}{\sm}}^t \breg(\xstar,x_0).
    \end{align*}
    Returning to
    \begin{align*}
        f(x_t)
        &\leq f(x) + (\sm-\sc) \breg(x,x_{t-1}) - \sm \breg(x,x_{t}),
    \end{align*}
    it follows by induction that
    \begin{align*}
        \sum_{i=1}^t \brk*{\frac{\sm}{\sm-\sc}}^i \brk*{f(x_i)-f(x)} \leq \sm \breg(x,x_0) - \brk*{\frac{\sm}{\sm-\sc}}^t \sm \breg(x,x_t).
    \end{align*}
    Since $\brk[c]*{f(x_{t})}_t$ is monotonically decreasing,
    \begin{align*}
        \sum_{i=1}^t \brk*{\frac{\sm}{\sm-\sc}}^i \brk*{f(x_i)-f(x)}
        &\geq \sum_{i=1}^t \brk*{\frac{\sm}{\sm-\sc}}^i \brk*{f(x_t)-f(x)} \\
        &= \frac{\sm \brk*{1+\frac{\sc}{\sm-\sc}}^t-1}{\sc} \brk*{f(x_t)-f(x)}.
    \end{align*}
    Thus,
    \begin{align*}
        \frac{\sm \brk*{1+\frac{\sc}{\sm-\sc}}^t-1}{\sc} \brk*{f(x_t)-f(x)}
        &\leq \sm \breg(x,x_0) - \brk*{\frac{\sm}{\sm-\sc}}^t \sm \breg(x,x_t) \\
        &\leq \sm \breg(x,x_0).
    \end{align*}
    Rearranging the terms yields the convergence result.
\end{proof}
\section{On Smoothness of \texorpdfstring{$\ell_p$}{lp} Regularization} \label{sec:non_smooth_regulator}
The following lemma indicate that $\pnorm{\cdot}^2$ for $1 < p < 2$ is not smooth. Hence, using it as a regularization can impair a smoothness assumption.
\begin{lemma} \label{lemma:non_smooth_regulator}
    Let $f(x)=\pnorm{x}^2$ for $1 < p < 2$ over $\xdomain=\R^d$ ($d>1$). Then $f(x)$ is not smooth with respect to $\pnorm{\cdot}$.
\end{lemma}

\begin{proof}
Assume by contradiction that $\reg(x)$ is $\sm$-smooth with respect to $\pnorm{\cdot}$ for some $\sm>0$.
Let $x=(1,0,0, \dots ,0)^T$ and $y=(1,\epsilon,0,\dots,0)^T$ for some $\epsilon>0$.
Thus,
\begin{align*}
    f(y)-f(x)-\nabla f(x) \cdot (y-x)
    &= (1+\epsilon^p)^{2/p}-1-\nabla f(x) \cdot (y-x) \\
    &= (1+\epsilon^p)^{2/p}-1-(1,0,\dots,0)^T \cdot (0,\epsilon,0,\dots,0)^T \\
    &= (1+\epsilon^p)^{2/p}-1.
\end{align*}
Using the identity $(1+x)^\gamma \geq 1+\gamma x$ for $x \geq 0$ and $\gamma \geq 1$, $(1+\epsilon^p)^{2/p}-1 \geq \tfrac{2}{p}\epsilon^p$. Since $p<2$,
\begin{align*}
    \lim_{\epsilon \rightarrow 0^+}\frac{\epsilon^2}{\epsilon^p}=0.
\end{align*}
We can pick sufficiently small $\epsilon$ with
\begin{align*}
    \frac{\epsilon^2}{(1+\epsilon^p)^{2/p}-1} &\leq \frac{p \epsilon^2}{2 \epsilon^{p}} < \frac{2}{\sm},
\end{align*}
for which
\begin{align*}
    f(y)-f(x)-\nabla f(x) \cdot (y-x)
    &= (1+\epsilon^p)^{2/p}-1
    > \frac{\sm \epsilon^2}{2}
    = \frac{\sm}{2}\pnorm{y-x}^2,
\end{align*}
and we get a contradiction to the smoothness assumption.
\end{proof}